\documentclass[10pt,lettersize,journal]{IEEEtran}

\usepackage{graphicx}
\usepackage{hyperref}
\usepackage{url}
\usepackage{amsthm}
\usepackage{algorithm}
\usepackage{algpseudocode} 
\usepackage{dsfont}
\usepackage{multirow, makecell}
\usepackage{xcolor}
\usepackage{amsmath,amsfonts,bm}

\newcommand{\colorname}{black}
\newtheorem{definition}{Definition}
\newtheorem{remark}{Remark}

\newtheorem{theorem}{Theorem}
\newtheorem{corollary}{Corollary}

\title{An Upper Bound for the Distribution Overlap Index and Its Applications}

\author{
        Hao~Fu,
        Prashanth~Krishnamurthy,~\IEEEmembership{Member,~IEEE,}
        Siddharth~Garg,~\IEEEmembership{Member,~IEEE,}
        and~Farshad~Khorrami,~\IEEEmembership{Senior~Member,~IEEE}
\thanks{The authors are with the Department of Electrical and Computer Engineering, NYU Tandon School of Engineering, Brooklyn NY, 11201. E-mail: \{ hf881@nyu.edu, pk929@nyu.edu, sg175@nyu.edu, khorrami@nyu.edu\}}}

\begin{document}

\maketitle

\begin{abstract}

This paper proposes an easy-to-compute upper bound for the  overlap index between two probability distributions without requiring any knowledge of the distribution models. The computation of our bound is time-efficient and memory-efficient and only requires finite samples. The proposed bound shows its value in one-class classification and domain shift analysis. Specifically, in one-class classification, we build a novel one-class classifier by converting the bound into a confidence score function. Unlike most one-class classifiers, the training process is not needed for our classifier. Additionally, the experimental results show that our classifier \textcolor{\colorname}{can be accurate with} only a small number of in-class samples and outperforms many state-of-the-art methods on various datasets in different one-class classification scenarios.  In domain shift analysis, we propose a theorem based on our bound. The theorem is useful in detecting the existence of domain shift and inferring data information. The detection and inference processes are both computation-efficient and memory-efficient. Our work shows significant promise toward broadening the applications of overlap-based metrics.

\end{abstract}

\section{Introduction}

Machine learning models have been utilized in various applications \cite{fkk20,pfkhk23,pfkk23,sfkgk23}. Measuring the similarity between distributions is a popular topic in the machine learning study. However, compared to other metrics that measure the distribution similarity, the literature on the distribution overlap index is thin. This paper contributes to the study of the overlap index which is defined as the area intersected by two probability density functions shown in Fig.~\ref{fig:example}(a). Specifically, unlike most related works that  have strong distribution assumptions (e.g., symmetry or unimodality), this work proposes an upper bound for the overlap index with distribution-free settings, in which the probability distributions are unknown.  The bound is easy to compute and contains two key terms: the norm of the difference between the two distributions' means and a variation distance between the two distributions over a subset. The computation of the bound is time-efficient and memory-efficient and requires only finite samples. Even though finding such an upper bound for the distribution overlap index is already valuable, we  further explore two additional applications of our bound as discussed below to broaden the applications of overlap-based metrics. 

One application of our bound is for one-class classification. Specifically, one-class classification refers to a model that outputs positive for in-class samples and negative for out-class samples that are  absent, poorly sampled, or not well defined (i.e., Fig.~\ref{fig:example}(b)). We proposed a novel one-class classifier by converting our bound into a confidence score function to evaluate if a sample is in-class or out-class. The proposed classifier has many advantages. For example, implementing deep neural network-based classifiers requires training thousands of parameters and large memory, whereas implementing our classifier does not require the training process. It only needs sample norms to calculate the confidence score. Besides, deep neural network-based classifiers need relatively large amounts of data to avoid under-fitting or over-fitting, whereas our method \textcolor{\colorname}{is empirically accurate with} only a small number of in-class samples. Therefore, our classifier is computation-efficient, memory-efficient, and data-efficient. Additionally, compared with other traditional one-class classifiers, such as Gaussian distribution-based classifier, Mahalanobis distance-based classifier \cite{LLLS18} and one-class support vector machine \cite{SPSSW01}, our classifier is distribution-free, explainable, and easy to understand. The experimental results show that the proposed one-class classifier outperforms many state-of-the-art methods on various datasets in different one-class classification scenarios.

Another application of our bound is for domain shift analysis. Specifically, a domain shift is a change in the dataset distribution between a model's training dataset and the testing dataset encountered during implementation (i.e., the overlap index value between the distributions of the two datasets is less than 1).  We proposed a theorem to calculate the model's testing accuracy in terms of the overlap index between the distributions of the training and testing datasets and further found the upper limit of the accuracy using our bound for the overlap index. The theorem is useful in detecting the existence of domain shift and inferring data information. Additionally, the detection and inference processes using the proposed theorem are both computation-efficient and memory-efficient and do not require the training process.

\begin{figure*}
    \centering
\includegraphics[width=0.85\textwidth,clip=true,trim=0in 0.15in 0in 0in]{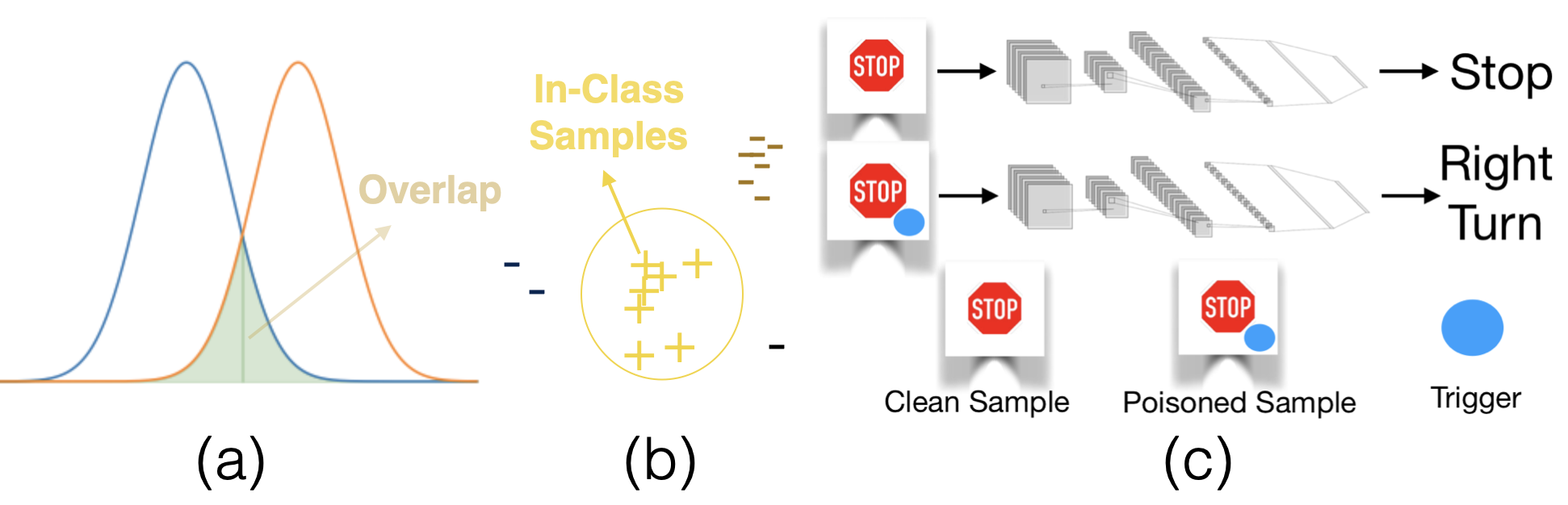}
    \caption{(a): Overlap of two distributions. (b): One-class classification. (c): Backdoor attack.}
    \label{fig:example}
\end{figure*}

Overall, the contributions of this paper include:
\begin{itemize}
    \item Deriving a distribution-free upper bound for the  overlap index.
    \item Building a novel one-class classifier with the bound being the confidence score function.
    \item Showing the outperformance of the one-class classifier over various state-of-the-art methods on several datasets, including UCI datasets, CIFAR-100, sub-ImageNet, etc., and in different one-class classification scenarios, such as novelty detection, anomaly detection, out-of-distribution detection, and neural network backdoor detection.
    \item Proposing a theorem based on the derived bound to detect the existence of domain shift and infer useful data information.
\end{itemize}

\subsection{Background and Related Works}

{\bf Measuring the similarity between distributions}: \cite{GL43} and \cite{W70} introduced the concept of the distribution overlap index. Other measurements for the similarity between distributions include the total variation distance, Kullback-Leibler divergence \cite{KL51}, Bhattacharyya's distance \cite{B43}, and Hellinger distance \cite{H09}.  In psychology,  some effect size measures' definitions involve the concept of the distribution overlap index, such as Cohen's U index \cite{C13}, McGraw and Wong's CL measure \cite{MW92}, and Huberty's I degree of non-overlap index \cite{HL00}. However, they all have strong distribution assumptions (e.g., symmetry or unimodality) regarding the overlap index. \cite{PC19} approximates the overlap index via kernel density estimators.

{\bf One-class classification}: \cite{MH96} coined the term one-class classification. One-class classification intersects with novelty detection, anomaly detection \cite{fkk25}, out-of-distribution detection \cite{fpkk24}, and outlier detection. \cite{YZLL21} explains the differences among these detection areas. \cite{KM14} discusses many traditional non neural network-based one-class classifiers, such as one-class support vector machine \cite{SPSSW01}, decision-tree \cite{CDGL99}, and one-class nearest neighbor \cite{T02}. Two neural network-based one-class classifiers are \cite{RVGDSBMK18} and OCGAN \cite{PNX19}.  \cite{ML22} introduces a Gaussian mixture-based energy measurement and compares it with several other score functions, including maximum softmax score \cite{DK17}, maximum Mahalanobis distance \cite{LLLS18}, and energy score \cite{LWOL20} for one-class classification. 

{\bf Neural network backdoor attack and detection}:  \cite{GLDG19} and \cite{LMALZWZ17} mentioned the concept of the neural network backdoor attack. The attack contains two steps: during training, the attacker injects triggers into the training dataset; during testing, the attacker leads the network to misclassify by presenting the triggers (i.e., Fig.~\ref{fig:example}(c)).   The data poisoning attack \cite{BNL12} and adversarial attack \cite{GPMXWOCB14} overlap with the backdoor attack. Some proposed trigger types are Wanet \cite{TA21}, invisible sample-specific \cite{LLWLHL21}, smooth \cite{ZPMJ21}, and reflection \cite{LMBL20}.  Some methods protecting neural networks from backdoor attacks include neural cleanse \cite{WYSLVZZ19}, fine-pruning \cite{LDG18}, and STRIP \cite{GXWCRN19}. NNoculation \cite{VLTKKKDG21} and RAID \cite{FVKGK22} utilize online samples to improve their detection methods. The backdoor detection problem \cite{FKGK23,f24,fskgk23} also intersects with one-class classification. Therefore, some one-class classifiers can  detect poisoned samples against the neural network backdoor attack.

\textbf{Organization of the Paper:}
We first provide preliminaries and derive the proposed upper  bound for the overlap index in Sec.~\ref{sec:Compute}. We next apply our   bound to domain shift analysis in Sec.~\ref{sec:Backdoor}. We then propose, analyze, and evaluate our novel one-class classifier in Sec.~\ref{sec:Classifier}.  We finally conclude the paper in Section~\ref{sec:Conclusion}.

\section{An Upper Bound for the Overlap Index}
\label{sec:Compute}
\subsection{Preliminaries}
For simplicity, we consider the $\mathbb{R}^n$ space and continuous random variables. \textcolor{\colorname}{We also define $P$ and $Q$ as two probability distributions in $\mathbb{R}^n$ with  $f_P$ and $f_Q$ being their probability density functions}.
\begin{definition}[Overlap Index]
\label{def:overlap}
The overlap $\eta: \mathbb{R}^n \times \mathbb{R}^n \to [0,1] $ of the two distributions is defined: 
\begin{align}
    \eta(P, Q) = \int_{\mathbb{R}^n} \min [f_P(x), f_Q(x)] dx.
    \label{eq:overlap}
\end{align}
\end{definition}

\begin{definition}[Total Variation Distance]
\label{def:total_variation}
    The total variation distance $\delta: \mathbb{R}^n \times \mathbb{R}^n \to [0,1] $ of the two distributions is defined as
    \begin{align}
        \delta(P, Q) = \frac{1}{2}\int_{\mathbb{R}^n} \left | f_P(x) - f_Q(x) \right | dx.
        \label{eq:total_variation}
    \end{align}
\end{definition}

\begin{definition}[Variation Distance on Subsets]
\label{def:subset}
    Given a subset $A$ from $\mathbb{R}^n$, we define $\delta_A: \mathbb{R}^n \times \mathbb{R}^n \to [0,1]$ to be the variation distance of the two distributions on $A$, which is
    \begin{align}
        \delta_A(P,Q) = \frac{1}{2}\int_A |f_P(x) - f_Q(x)| dx.
        \label{eq:subset}
    \end{align}
\end{definition}

\begin{remark}
\label{rem:convertible}
 One can prove that $\eta$ and $\delta$ satisfy the following equation:
\begin{align}
    \eta(P, Q) =  1 - \delta(P, Q) = 1 - \delta_A(P, Q) - \delta_{\mathbb{R}^n\setminus A}(P, Q).
    \label{eq:convertible}
\end{align}
The quantity $\delta_A$ defined in (\ref{eq:subset}) will play an important role in deriving our  upper bound for $\eta$.
\end{remark}


\subsection{The Upper Bound for the Overlap Index}
We now proceed with deriving our proposed upper bound.
\begin{theorem}
\label{thm:bound}
Without loss of generality, assume $D^+$ and $D^-$ are two probability distributions  on a bounded domain $B \subset \mathbb{R}^n$ with defined norm $||\cdot||$ \footnote{In this paper, we use the $L_2$ norm. However, the choice of the norm is not unique and the analysis can be carried out using other norms as well.} (i.e., $\sup_{x\in B} ||x|| < +\infty$), then for any subset $A \subset B$ with its complementary set $A^c = B \setminus A$, we have
\begin{align}
    \eta(D^+, D^-) \le 1 - \frac{1}{2r_{A^c}} ||\mu_{D^+} - \mu_{D^-}|| - \frac{r_{A^c}-r_A}{r_{A^c}}\delta_{A}
    \label{eq:theorem}
\end{align} where $r_A = \sup_{x\in A} ||x||$ and $r_{A^c} = \sup_{x \in A^c}||x||$, $\mu_{D^+}$ and $\mu_{D^-}$ are the means of $D^+$ and $D^-$, and $\delta_A$ is the variation distance on set $A$ defined in {\bf Definition}~\ref{def:subset}. Moreover, let $r_B = \sup_{x\in B} ||x||$,  then we have
\begin{align}
    \eta(D^+, D^-) \le 1 - \frac{1}{2r_{B}} ||\mu_{D^+} - \mu_{D^-}|| - \frac{r_{B}-r_A}{r_{B}}\delta_{A}.
    \label{eq:theorem_special}
\end{align} Since (\ref{eq:theorem_special}) holds for any $A$, a tighter bound can be written as
\begin{align}
    \eta(D^+, D^-) \le 1 - \frac{1}{2r_{B}} ||\mu_{D^+} - \mu_{D^-}|| - \max_A \frac{r_{B}-r_A}{r_{B}}\delta_{A}.
    \label{eq:theorem_tight}
\end{align}
\end{theorem}
\begin{proof} 
See Appendix~\ref{app:proof}.
\end{proof}
\begin{remark}
    \label{rem:theorem_assumption}
   The only assumption in this theorem is that the probability distribution domain is bounded. However, almost all real-world applications satisfy the boundedness assumption  since the data is bounded. Therefore, $r_B$ can always be found (or at least a reasonable approximation can be found). Additionally, we can constrain $A$ to be a bounded ball so that $r_A$ is also known. Although the proof of this theorem involves probability density functions, the computation does not require knowing the probability density functions but only finite samples because  we can use the law of large numbers to estimate $||\mu_{D^+} - \mu_{D^-}||$ and $\delta_A$, which will be shown next.
\end{remark}

\subsection{Approximating the Bound with Finite Samples}
Let $g: B \to \{0, 1\}$ be a condition function\footnote{The condition function is an indicator function $\mathds{1}\{condition\}$ that outputs 1 when the input satisfies the given condition and 0 otherwise.} and define $A = \{ x~|~ g(x) = 1, x \in B\}$.   According to the definition of $\delta_A$ and triangular inequality, it is trivial to prove that
\begin{align}
    \delta_{A}(D^+, D^-) 
     \ge \frac{1}{2} \left | \mathbb{E}_{D^+}[g] - \mathbb{E}_{D^-}[g] \right |.
    \label{eq:monte_carlo}
\end{align}
Calculating $\mathbb{E}_{D^+}[g]$ and $\mathbb{E}_{D^-}[g]$ is easy: one just needs to draw samples from $D^+$ and $D^-$, and then average their $g$ values. Applying (\ref{eq:monte_carlo}) into {\bf Theorem}~\ref{thm:bound} gives the following corollary:
\begin{corollary}
\label{cor:bound}
Given $D^+$, $D^-$, $B$, and $||\cdot||$ used in {\bf Theorem}~\ref{thm:bound}, let $A(g) = \{ x~|~ g(x) = 1, x \in B\}$ with any condition function $g: B \to \{0, 1\}$. Then, an upper bound for $\eta(D^+, D^-)$ that can be obtained by our approximation is
\begin{align}
    \eta(D^+, D^-) &\le 1 - \frac{1}{2r_{B}} ||\mu_{D^+} - \mu_{D^-}|| - \\
    &\max_g \frac{r_{B}-r_{A(g)}}{2r_{B}}\left | \mathbb{E}_{D^+}[g] - \mathbb{E}_{D^-}[g] \right |.
    \label{eq:corollary_tight}
\end{align}
\end{corollary}
Given several condition functions $\{g_j\}_{j=1}^k$ and finite sample sets (i.e., $\{x_i^+\}_{i=1}^n \sim D^+$ and $\{x_i^-\}_{i=1}^m \sim D^-$), Alg.~\ref{alg:bound} shows how to compute  the RHS of (\ref{eq:corollary_tight}). 

 \begin{algorithm}
\caption{ComputeBound($\{x_i^+\}_{i=1}^n$, $\{x_i^-\}_{i=1}^m$, $\{g_j\}_{j=1}^k$)}  
 \label{alg:bound} 
 \begin{algorithmic}
 \State $B \leftarrow \{x_1^+, x_2^+, ..., x_n^+, x_1^-, x_2^-, ..., x_m^-\}$ and  $r_B \leftarrow \max_{x \in B} ||x||$
  \State $\Delta_\mu \leftarrow \left| \left| \frac{1}{n}\sum_{i =1}^n x_i^+ -\frac{1}{m}\sum_{i=1}^m x_i^- \right| \right| $
     \For {$j = 1 \to k$}
     \State $A = \{x~|~g_j(x)=1, x \in B\}$ and $r_A \leftarrow \max_{x\in A}||x||$
        \State  $s_j \leftarrow \left (1 - \frac{r_A}{r_B} \right)\left | \frac{1}{n}\sum_{i=1}^n g_j(x_i^+) -  \frac{1}{m}\sum_{i=1}^m g_j(x_i^-)\right | $
     \EndFor
     \State {\bf Return:} $1 - \frac{1}{2r_B}\Delta_\mu - \frac{1}{2} \max_j s_j$
 \end{algorithmic}
 \end{algorithm}

 \begin{remark}
     \label{rem:alg}
     The choice of condition functions is not unique. In this work, we use the indicator function $g(x) = \mathds{1}\{ ||x|| \le r\}$, which outputs $1$ if $||x|| \le r$ and 0 otherwise. By setting different values for $r$, we generate a family of condition functions. \textcolor{\colorname}{The motivation for choosing such  indicator function form is that it is the most simple way to separate a space nonlinearly and it saves computations by directly applying $r$ into {\bf Corollary~\ref{cor:bound}}. However, other indicator functions, such as RBF kernel-based indicator functions, are worth exploring and will be considered in our future works.}
 \end{remark}

\section{Application of Our Bound to One-Class Classification}
\label{sec:Classifier}

\subsection{Problem Formulation for One-Class Classification}
 Given $\mathbb{R}^d$ space and $n$ samples $\{x_i\}_{i=1}^n$ that lie in an unknown probability distribution, we would like to build a test $\Psi: \mathbb{R}^d \to \{\pm 1\}$ so that for any new input $x$, $\Psi(x)$ outputs $1$ when $x$ is from the same unknown probability distribution, and outputs -$1$, otherwise. Some applications of $\Psi$ are  novelty detection,   out-of-distribution detection, and backdoor detection (e.g., Fig.~\ref{fig:example}(b, c)). 

 \subsection{A Novel Confidence Score Function}
 Given some in-class samples $\{x_i\}_{i=1}^n$, one can pick several condition functions $\{g_j\}_{j=1}^k$, where $g_j(x)=\mathds{1}\{||x||\le r_j\}$ for different $r_j$, so that $f(x) = ComputeBound(\{x\}, \{x_i\}_{i=1}^n, \{g_j\}_{j=1}^k)  $ defined in Alg.~\ref{alg:bound} is a score function that measures the likelihood of any input, $x$, being an in-class sample.  Alg.~\ref{alg:classifier} shows the overall one-class classification algorithm. 

 \begin{algorithm}
\caption{The Novel One-Class Classifier for the Input $x$}  
 \label{alg:classifier} 
 \begin{algorithmic}
 \State Given in-class samples $\{x_i\}_{i=1}^n$, select several condition functions $\{g_j\}_{j=1}^k$, set a threshold $T_0$
 \If {$ComputeBound(\{x\}, \{x_i\}_{i=1}^n, \{g_j\}_{j=1}^k)\ge T_0$}
 \State $x$ is an in-class sample
 \Else 
 \State $x$ is an out-class sample
 \EndIf
 \end{algorithmic}
 \end{algorithm}

 \begin{remark}
     \label{rem:classifier}
       The score function $f$ measures the maximum similarity between the new input $x$ and the available in-class samples $\{x_i\}_{i=1}^n$. Different $T_0$ lead to different detection accuracy. However, we will show that the proposed one-class classifier has an overall high accuracy under different $T_0$. 
 \end{remark}
\subsection{ Computation and Space Complexities}
 Our algorithm can pre-compute and store $\frac{1}{n}\sum_{i=1}^n x_i$ and $\frac{1}{n}\sum_{i=1}^n g_j(x_i)$ with $j=1,2,...,k$.  Therefore, the total space complexity  is $\mathcal{O}(k+1)$. Assume that the total number of new online inputs is $l$; then, for every new input $x$, our algorithm needs to calculate $||x||$ once and $s_j$ for $k$ times. Therefore, the total computation complexity is $\mathcal{O}(l(k+1))$. Empirically, we restricted $k$ to be a small number (e.g., 10) so that even devices  without strong computation power can  run our algorithm efficiently.  Therefore,  our classifier is computation-efficient and memory-efficient. 

 \subsection{Evaluation}

\textcolor{\colorname}{{\bf Overall Setup:} We used $r_j = \frac{j}{k}r_B$ with $k=50$ and $r_B$ defined in {\bf Theorem~\ref{thm:bound}} unless specified. All the baseline algorithms with optimal hyperparameters, related datasets, and models were acquired from the corresponding authors' websites. The only exception is backdoor detection, in which we created our own models. However, we have carefully fine-tuned the baseline methods' hyperparameters to ensure their best performance over other hyperparameter choices.}

\subsubsection{ Novelty Detection}
We evaluated our classifier on 100 small UCI datasets \cite{UCI,DG19} and recorded the area under the receiver operating characteristic curve (AUROC).  Fig.~\ref{fig:errorbar} shows the mean and standard deviation of AUROC for ours and other classifiers. \textcolor{\colorname}{Detailed numerical results are in Table~\ref{tab:novelty} in Appendix~\ref{app:novelty}.} The implementation code is provided in the supplementary material. 

Our classifier outperforms all the baseline methods by showing the highest average and lowest standard deviation of AUROC. Besides the results, our classifier is distribution-free, computation-efficient, and memory-efficient, whereas some other classifiers do not. Our method is also easy to explain and understand: the score  measures the maximum similarity between the new input and the available in-class samples. Therefore, we conclude  that our classifier is valid for novelty detection.

\begin{figure}
    \centering
    \includegraphics[width=\linewidth]{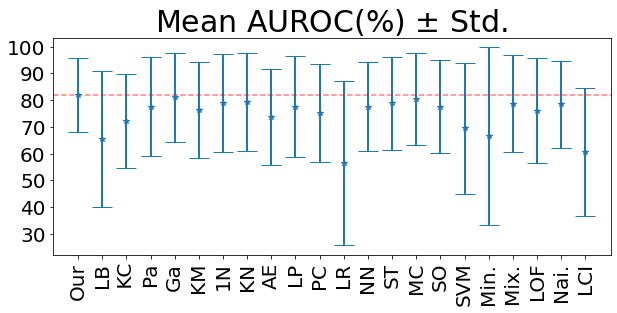}
    \caption{Evaluation on 100 small UCI datasets for novelty detection. Details are in Table~\ref{tab:novelty}.}
    \label{fig:errorbar}
\end{figure}

\subsubsection{ Out-of-Distribution Detection}
We used CIFAR-10 and CIFAR-100 \textcolor{\colorname}{testing datasets} \cite{KH09} as the in-distribution datasets. The compared methods contain MSP \cite{DK17}, Mahalanobis \cite{LLLS18}, Energy score \cite{LWOL20}, and GEM \cite{ML22}. We used WideResNet \cite{ZK16} to extract features from the raw data. \textcolor{\colorname}{The WideResNet models (well-trained on CIFAR-10 and CIFAR-100 training datasets) and corresponding feature extractors were acquired from \cite{ML22}. All the methods were evaluated  in the same feature spaces with their optimal hyperparameters for fair comparisons.} To fit the \textcolor{\colorname}{score function's} parameters for all the methods, we formed a small dataset by randomly selecting 10 samples from each class. The out-of-distribution datasets include Textures \cite{CMKMV14}, SVHN \cite{NWCBWN11}, LSUN-Crop \cite{YSZSFX15}, LSUN-Resize \cite{YSZSFX15}, and iSUN \cite{XEZFKX15}.  We used three metrics: the detection accuracy for out-of-distribution samples when the detection accuracy for in-distribution samples is $95\%$ (TPR95), AUROC, and area under precision and recall (AUPR). 

\begin{table*}
    \centering
    \caption{\textcolor{\colorname}{Average performance on various out-of-distribution datasets. Details are in Tables~\ref{tab:out-of-distribution_CIFAR10} and \ref{tab:out-of-distribution_CIFAR100}.}}
    \begin{tabular}{c|cccccc}
    \hline
     In-Distributions & Method & TPR95  & AUROC  & AUPR  & \textcolor{\colorname}{Time/Sample}  & \textcolor{\colorname}{Memory}   \\
        \hline
   \multirow{5}{*}{CIFAR-10} &   Ours & {\bf 81.39\%} & {\bf 96.08\%} & 96.17\%  & \textcolor{\colorname}{3.0ms} & \textcolor{\colorname}{{\bf 1048.22MiB}} \\
     & MSP & 50.63\% & 91.46\% & {\bf 98.07\%} & \textcolor{\colorname}{{\bf 0.02ms}} &  \textcolor{\colorname}{1825.21MiB} \\
     & Mahala. & 46.83\% & 90.46\% & 97.92\%  & \textcolor{\colorname}{30.61ms} & \textcolor{\colorname}{1983.17MiB} \\
     & Energy  & 68.31\% & 92.32\% & 97.96\% & \textcolor{\colorname}{0.22ms} & \textcolor{\colorname}{1830.01MiB} \\
     & GEM & 50.81\% & 90.45\% & 97.91\% & \textcolor{\colorname}{25.62ms} & \textcolor{\colorname}{1983.51MiB} \\
     \hline
     \multirow{5}{*}{CIFAR-100}    & Ours & {\bf 72.93\%} & {\bf 93.53\%} & 93.69\% & \textcolor{\colorname}{3.0ms} & \textcolor{\colorname}{{\bf 1134.32MiB}} \\
     & MSP & 19.87\% & 75.97\% & 94.09\% & \textcolor{\colorname}{{\bf 0.02ms}} & \textcolor{\colorname}{1825.98MiB} \\
     & Mahala. & 48.35\% & 84.90\% & 96.37\% & \textcolor{\colorname}{56.24ms} & \textcolor{\colorname}{1983.81MiB} \\
     & Energy  & 27.91\% & 80.44\% & 95.15\% & \textcolor{\colorname}{0.21ms} & \textcolor{\colorname}{1838.23MiB} \\
     & GEM &  48.36\% & 84.94\% & {\bf 96.38\%} & \textcolor{\colorname}{56.27ms} & \textcolor{\colorname}{1984.77MiB} \\
     \hline
    \end{tabular}
    \label{tab:ODDave}
\end{table*}

\textcolor{\colorname}{Table~\ref{tab:ODDave} shows the average results for CIFAR-10  and  CIFAR-100. The details for each individual out-of-distribution dataset can be found in Tables~\ref{tab:out-of-distribution_CIFAR10} and \ref{tab:out-of-distribution_CIFAR100} in Appendix~\ref{app:odd}. }
Our method outperforms the other methods by \textcolor{\colorname}{using the least memory} and showing the highest TPR95 and AUROC on average. The AUPR of our approach is in the same range as other baseline methods and the execution time of our approach for each sample is around 3 milliseconds (ms).  We further evaluated our approach by using different numbers of indicator functions $k$ and plotted the results in Fig.~\ref{fig:improve}. From the figure, the performance of our approach increases with more indicator functions being used and eventually converges to a limit. This limit is determined by the out-of-distribution dataset, the tightness of our bound in {\bf Corollary}~\ref{cor:bound}, and the form of utilized indicator functions.

\begin{figure*}
    \centering
    \includegraphics[width=0.9\textwidth]{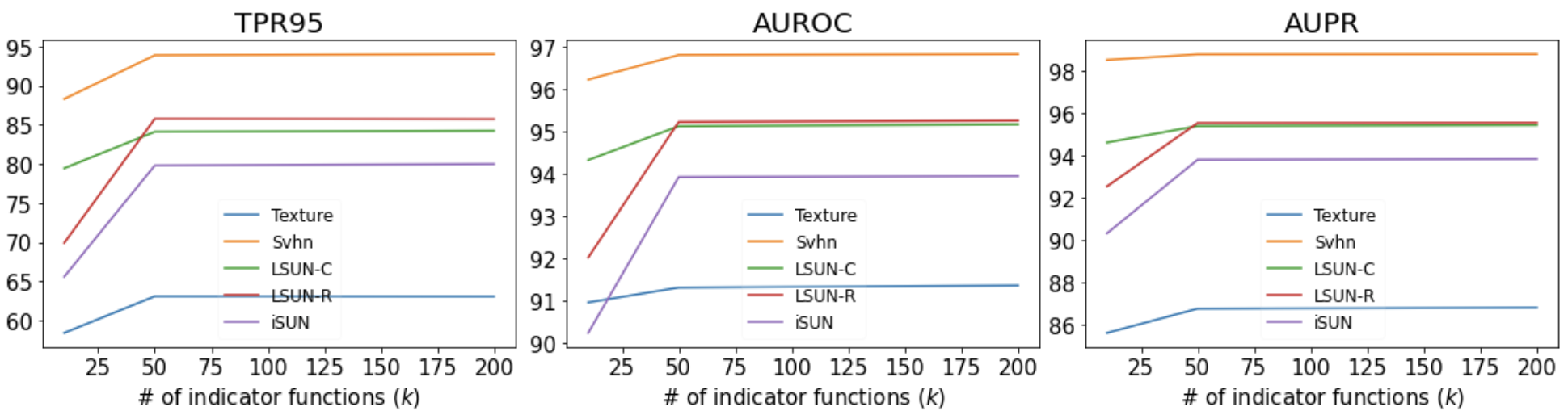}
    \caption{\textcolor{\colorname}{Performance of our approach with different $k$ when CIFAR-10 is the in-distribution data.}}
    \label{fig:improve}
\end{figure*}

We also evaluated our approach with balanced datasets because the above out-of-distribution datasets are much different in size from the in-distribution datasets.  \textcolor{\colorname}{On balanced datasets, our approach shows higher AUPR than the baseline methods as shown in Table~\ref{tab:AVEbackdoor} with details in Table~\ref{tab:backdoor} in Appendix~\ref{app:backdoor}.  We also empirically observed that the compared baseline methods reported errors when data dimensions are dependent because the compared baseline methods need to calculate the inverse of the estimated covariance matrices that will not be full rank if data dimensions are dependent. We have reported this observation in Table~\ref{tab:backdoor} in the appendix. In contrast, our approach works since it does not require finding the inverse of any matrices. Further, Table~\ref{tab:ODDave} and Table~\ref{tab:AVEbackdoor} together show that the baseline methods perform well only for out-of-distribution detection, whereas our approach performs well for both out-of-distribution detection and backdoor detection (details are explained in next subsection). }
In summary, our classifier is a valid out-of-distribution detector.

\subsubsection{ Backdoor Detection}

\begin{figure}
    \centering
    \includegraphics[width=\linewidth]{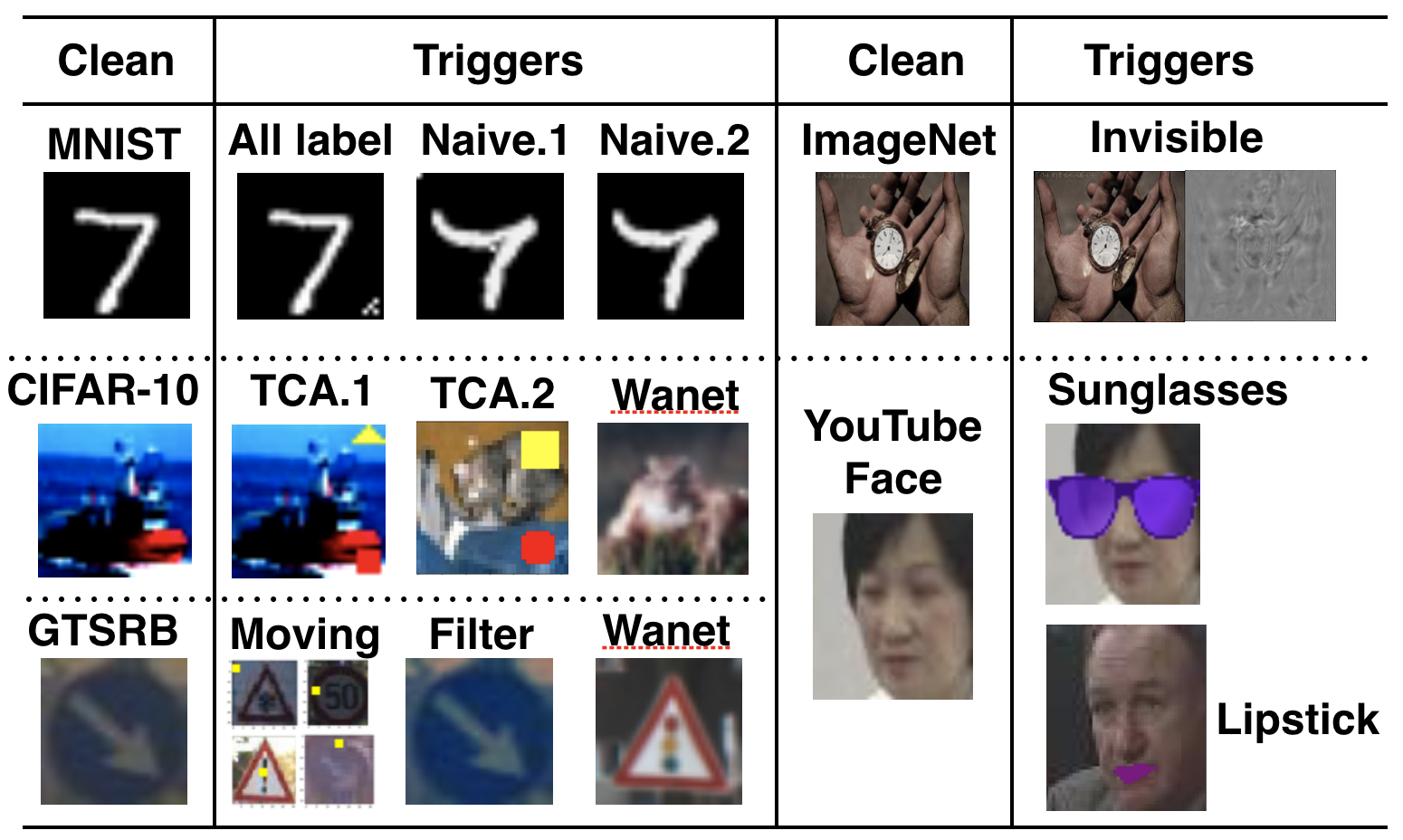}
    \caption{ \textcolor{\colorname}{Pictures under "Triggers" are poisoned samples regarding different backdoored attacks. Pictures under "Clean" are clean samples for each dataset.}}
    \label{fig:trigger}
\end{figure}

The utilized datasets are MNIST \cite{LCB10}, CIFAR-10 \cite{KH09}, GTSRB \cite{SSSI11}, YouTube Face \cite{WHM11}, and sub-ImageNet \cite{JWRLKL09}. The adopted backdoor attacks include naive triggers, all-label attacks \cite{GLDG19}, moving triggers \cite{FVKGK22}, Wanet \cite{TA21}, combination attacks, large-sized triggers, filter triggers, and invisible sample-specific triggers \cite{LLWLHL21}, \textcolor{\colorname}{as listed in Fig.~\ref{fig:trigger}.}  The neural network architecture includes Network in Network \cite{LCY13},   Resnet \cite{HZRS16}, and other networks from \cite{WYSLVZZ19,GLDG19}.  For each backdoor attack, we assume that a small clean validation dataset is available (i.e., 10 samples from each class) at the beginning. Therefore, the poisoned samples (i.e., samples attached with triggers) can be considered out-class samples, whereas the clean samples can be considered in-class samples. We used the backdoored network to extract data features. Then, we evaluated our one-class classifier and compared it with \textcolor{\colorname}{the previous baseline methods} and STRIP \cite{GXWCRN19} in the feature space. The metrics used are the same: TPR95 (i.e., the detection accuracy for poisoned samples when the detection accuracy for clean samples is 95\%), AUROC, and AUPR. \textcolor{\colorname}{Table~\ref{tab:AVEbackdoor} shows the average performance. Details on each individual trigger can be found in Table~\ref{tab:backdoor} in Appendix~\ref{app:backdoor}}  

\begin{table}
    \centering
    \caption{\textcolor{\colorname}{Average performance for backdoor detection over various backdoor triggers and datasets.}}
    \begin{tabular}{c|ccccc}
        \hline
      Metrics (\%) &  Ours & STRIP & \textcolor{\colorname}{Mahalanobis} & \textcolor{\colorname}{GEM} & \textcolor{\colorname}{MSP}  \\
      \hline
        TPR95  & 89.40 & 39.60 & \textcolor{\colorname}{56.97} & \textcolor{\colorname}{{\bf 91.57}} & \textcolor{\colorname}{39.24} \\
      AUROC & {\bf 96.68} & 70.30 & \textcolor{\colorname}{75.94} & \textcolor{\colorname}{58.08} & \textcolor{\colorname}{54.92} \\
      AUPR & {\bf 95.42} & 68.76 & \textcolor{\colorname}{76.37} & \textcolor{\colorname}{75.88} & \textcolor{\colorname}{60.52} \\
     \hline
    \end{tabular}
    \label{tab:AVEbackdoor}
\end{table}

\textcolor{\colorname}{From the table, our classifier outperforms other baseline methods on average by showing the highest AUROC and AUPR. Additionally, our approach also has a very high TPR95.} For each individual trigger, the TPR95 of our method is over 96\% for most cases, the AUROC of our method is over 97\% for most cases, and the AUPR of our method is over 95\% for most cases. It is also seen that our classifier is robust against the latest or advanced backdoor attacks, such as Wanet,  invisible trigger, all label attack, and filter attack, whereas \textcolor{\colorname}{the baseline methods show low performance on those attacks.} Therefore, we conclude that our classifier is valid for backdoor detection.

\subsubsection{ Anomaly Detection with  Iterative Scores}
For an input $x$, denote the confidence score calculated using the condition functions $g_j(x) = \mathds{1}\{ ||x|| \le r_j\}$ as $s(x)$, then its iterative confidence score, $s^\prime(x)$, is calculated by Alg.~\ref{alg:bound} but with the new condition functions $g_j^\prime(x) = \mathds{1}\{ s(x) \le \frac{j}{k}\}$. To evaluate the efficacy of this iterative approach, we consider the comparison with deep network-based classifiers, including DCAE \cite{MF15}, AnoGAN \cite{SSWSL17}, Deep SVDD \cite{RVGDSBMK18}, and OCGAN \cite{PNX19}, for anomaly detection on the CIFAR-10 dataset in the raw image space. The anomaly detection setting considers one class as normal and the remaining classes as anomalous. The baseline methods need to first train deep networks on available normal-class samples and then use the trained models to detect anomalies, whereas our approach does not require this training process. Additionally, we empirically observed that the baseline methods require around 5000 available normal-class training samples to be effective, whereas our approach needs only 100 available normal samples to be effective. For a fair comparison, we allow the baseline methods to use 5000 available normal-class samples to train their deep networks to achieve their best performance. 

 Table~\ref{tab:anomaly} shows the results. Our approach shows the highest AUROC on average. Additionally, our approach shows either the best or the second-best performance for 9 individual cases. Moreover, if the baseline methods are trained with only 100 available normal samples, then their performance will become lower. For example, Deep SVDD shows only 61.1\% AUROC on average. Therefore, our approach is valid and sample-efficient in the raw image space and outperforms the deep network-based anomaly detection methods. 

\begin{table*}
    \centering
    \caption{AUROC (\%) on CIFAR-10 with different normal classes. {\bf Boldface} shows the best performing algorithm, whereas \underline{underline} shows the second best algorithm. }
    \begin{tabular}{c|cccccccccc|c}
        \hline
  \textcolor{\colorname}{Methods}  &     \textcolor{\colorname}{Air.}  &  \textcolor{\colorname}{Aut.} & \textcolor{\colorname}{Bird} & \textcolor{\colorname}{Cat} & \textcolor{\colorname}{Deer} & \textcolor{\colorname}{Dog} & \textcolor{\colorname}{Frog} & \textcolor{\colorname}{Hor.} & \textcolor{\colorname}{Ship} & \textcolor{\colorname}{Tru.} &\textcolor{\colorname}{Ave.} \\
      \hline
   Ours  &   \underline{75.0} &  {\bf 67.9}  &  \underline{56.9}  & \underline{60.4} & \underline{68.6} & \underline{63.1} & \underline{69.9} & 61.0 & \underline{77.9} & {\bf 76.0} & {\bf 67.7}  \\
 OCGAN  &    {\bf 75.7}  &  {53.1}  &  {\bf 64.0}  & {\bf 62.0} & {\bf 72.3} & {62.0} &  {\bf 72.3} & 57.5 & {\bf 82.0} & 55.4 & \underline{65.6} \\
 Deep SVDD &    {61.7}  &  \underline{65.9}  &  {50.8}  & {59.1} & {60.9} & {\bf 65.7} &  {67.7} & {\bf 67.3} & 75.9 & \underline{73.1} & 64.8  \\
 AnoGAN & 67.1 & 54.7 & 52.9 & 54.5 & 65.1 & 60.3 & 58.5 & \underline{62.5} & 75.8 & 66.5 & 61.8 \\
 DCAE & 59.1 & 57.4 & 48.9 & 58.4 & 54.0 & 62.2 & 51.2 & 58.6 & 76.8 & 67.3 & 59.4 \\
     \hline
    \end{tabular}
    \label{tab:anomaly}
\end{table*}

\section{Application of Our Bound to Domain Shift Analysis}
\label{sec:Backdoor}
Our bound is also useful in domain shift analysis.
\begin{theorem}
    \label{thm:classification_accuracy}
    Assume that $D$ and $D^*$ are two different data distributions (i.e., $\eta(D, D^*)<1$) and denote the overall accuracy of the model on $D^*$  as $Acc$. If a model is trained on $D$ with $p$ accuracy on \textcolor{\colorname}{$D$ and $q$ accuracy on $D^*\setminus D$,}  then we have
    \begin{align}
        Acc &\le (p-q)(1 - \frac{1}{2r_{B}} ||\mu_{D} - \mu_{D^*}|| - \\
        & \max_g \frac{r_{B}-r_{A(g)}}{2r_{B}}\left | \mathbb{E}_{D}[g] - \mathbb{E}_{D^*}[g] \right |) +q.
        \label{eq:shift}
    \end{align}
\end{theorem}
\begin{proof}
    See Appendix~\ref{app:proof2}.
\end{proof}
\begin{remark}
    \label{rem:backdoor}
As one case of the domain shift, a backdoor attack scenario (Fig.~\ref{fig:example}(b)) considers that the model has a zero accuracy on poisoned data distribution (i.e., $q=0$) as the attack success rate is almost 100\%. Define the clean data distribution as $D$, poisoned data distribution as $D^p$, and a testing data distribution $D^*$ composed by $D$ and $D^p$, i.e., $D^* = \sigma D + (1-\sigma)D^p$, where $\sigma \in [0,1]$ is the purity ratio (i.e., the ratio of clean samples to the entire testing samples). Then (\ref{eq:shift}) becomes 
\begin{align}
    Acc &\le p (1- \frac{1-\sigma}{2r_B}||\mu_D-\mu_{D^p}|| - \\ 
     & (1-\sigma) \max_g \frac{r_B-r_{A(g)}}{2r_B}|\mathbb{E}_{D}[g] - \mathbb{E}_{D^p}[g]|).
    \label{eq:bound2}
\end{align} 
\end{remark}


{\bf Experimental Illustration of Theorem~\ref{thm:classification_accuracy}}: The backdoor attack scenario is used to illustrate Theorem~\ref{thm:classification_accuracy} using MNIST, GTSRB, YouTube Face, and sub-ImageNet datasets and their domain-shifted versions shown in  ``All label'' for MNIST, ``Filter'' for GTSRB, ``Sunglasses'' for YouTube Face, and ``Invisible'' for sub-ImageNet in Fig.~\ref{fig:trigger}. We composed the testing datasets $D^*$ with $\sigma=0,0.1,...,0.9,1$ and calculated the RHS of (\ref{eq:bound2}) using $L_1$, $L_2$, and $L_\infty$ norms in the raw image space, model output space, and hidden layer space. Fig.~\ref{fig:domainshift} shows the actual model accuracy and corresponding upper bounds with different $\sigma$s. The actual model accuracy is below all the calculated upper limits, which validates {\bf Theorem}~\ref{thm:classification_accuracy}.

Additionally, the difference between actual model accuracy and the calculated upper bound accuracy can reflect the extent of the domain shift in the test dataset. From Fig.~\ref{fig:domainshift}, a large difference reflects a large domain shift. When the domain shift vanishes, the actual model accuracy and the calculated upper bound accuracy coincide. Compared to other methods, our approach is time-efficient.

{\bf Theorem~\ref{thm:classification_accuracy}} can also help infer useful data information to most likely find and remove domain-shifted samples to eliminate the domain shift. Specifically, from Fig.~\ref{fig:domainshift}, the calculated upper bound accuracy varies with norms and feature spaces. The inference is that a low calculated upper bound accuracy implies a high likelihood of distinguishing original and domain-shifted data distributions in that particular utilized feature space. For example, in Fig.~\ref{fig:domainshift} MNIST and YouTube, the input space with $L_\infty$ norm shows the lowest calculated upper bound accuracy. Therefore, the inference is that the original and domain-shifted data distributions are likely to be distinguished in the raw image space. Indeed, even vision inspection can easily distinguish original and domain-shifted samples for the MNIST and YouTube cases given their domain-shifted samples under ``All label'' and ``Sunglasses'' in Fig.~\ref{fig:trigger}. As for GTSRB and ImageNet, the input space has the highest upper limit lines. Therefore, the original and domain-shifted data distributions are less likely to be distinguished in the raw image space. Since their domain-shifted samples are ``Filter'' and ``Invisible'' in Fig.~\ref{fig:trigger}, the vision inspection barely discriminates between original and domain-shifted samples for GTSRB  and sub-ImageNet. Moreover, the hidden layer space gives the lowest estimated upper bound accuracy for GTSRB and ImageNet. Therefore, original and domain-shifted samples are more likely to be visually distinguished in the hidden layer space  as shown in Fig.~\ref{fig:feature}. 

\begin{figure}
    \centering
    \includegraphics[width=0.9\linewidth]{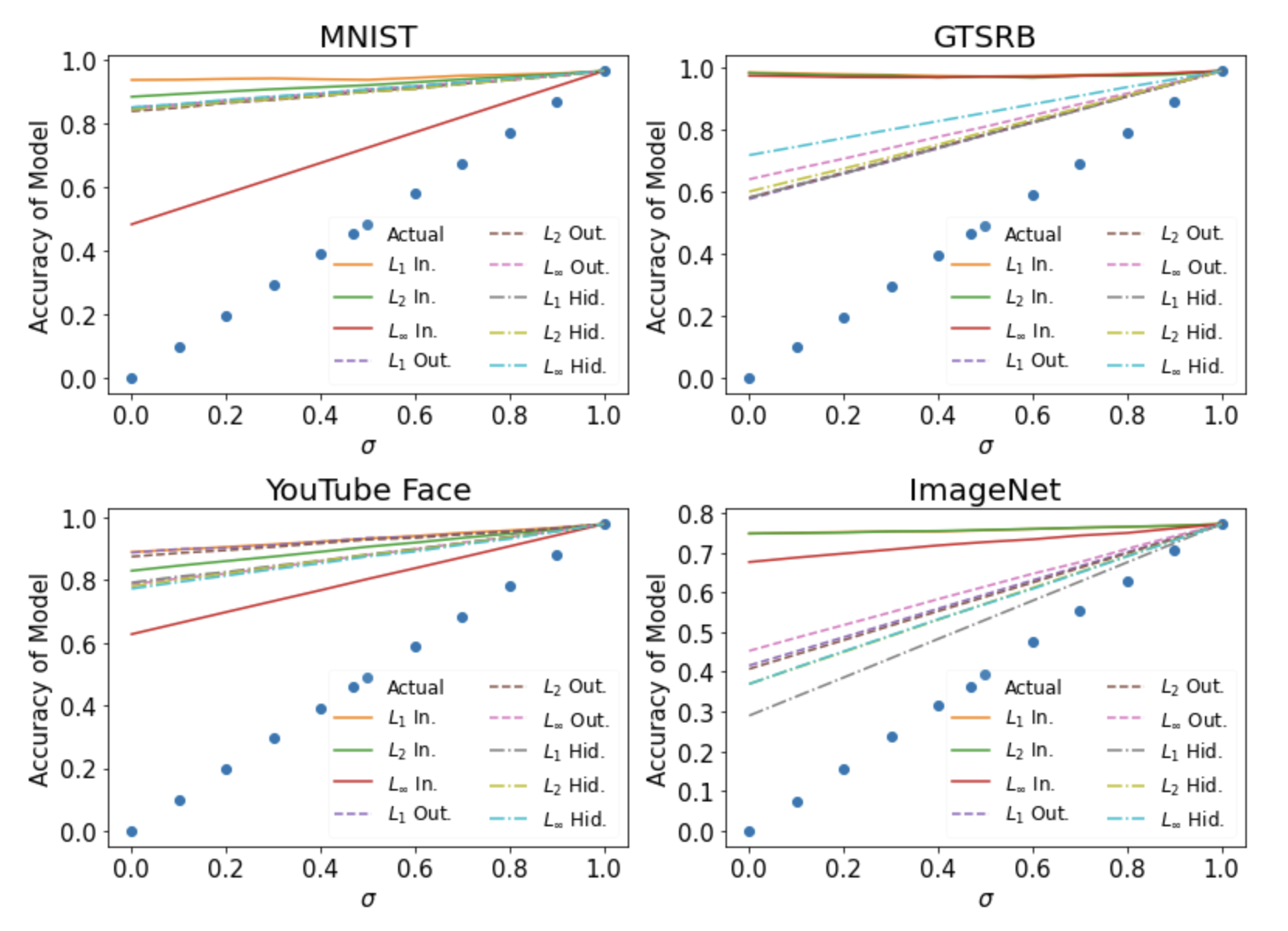}
    \caption{\textcolor{\colorname}{The actual model accuracy (dot) vs. (\ref{eq:shift}) (solid) calculated with $L_1$, $L_2$, and $L_\infty$ norms in input, output, and hidden spaces. x: the ratio of clean samples to the entire testing samples.}}
    \label{fig:domainshift}
\end{figure}

\begin{figure}
    \centering
    \includegraphics[width=0.6\linewidth]{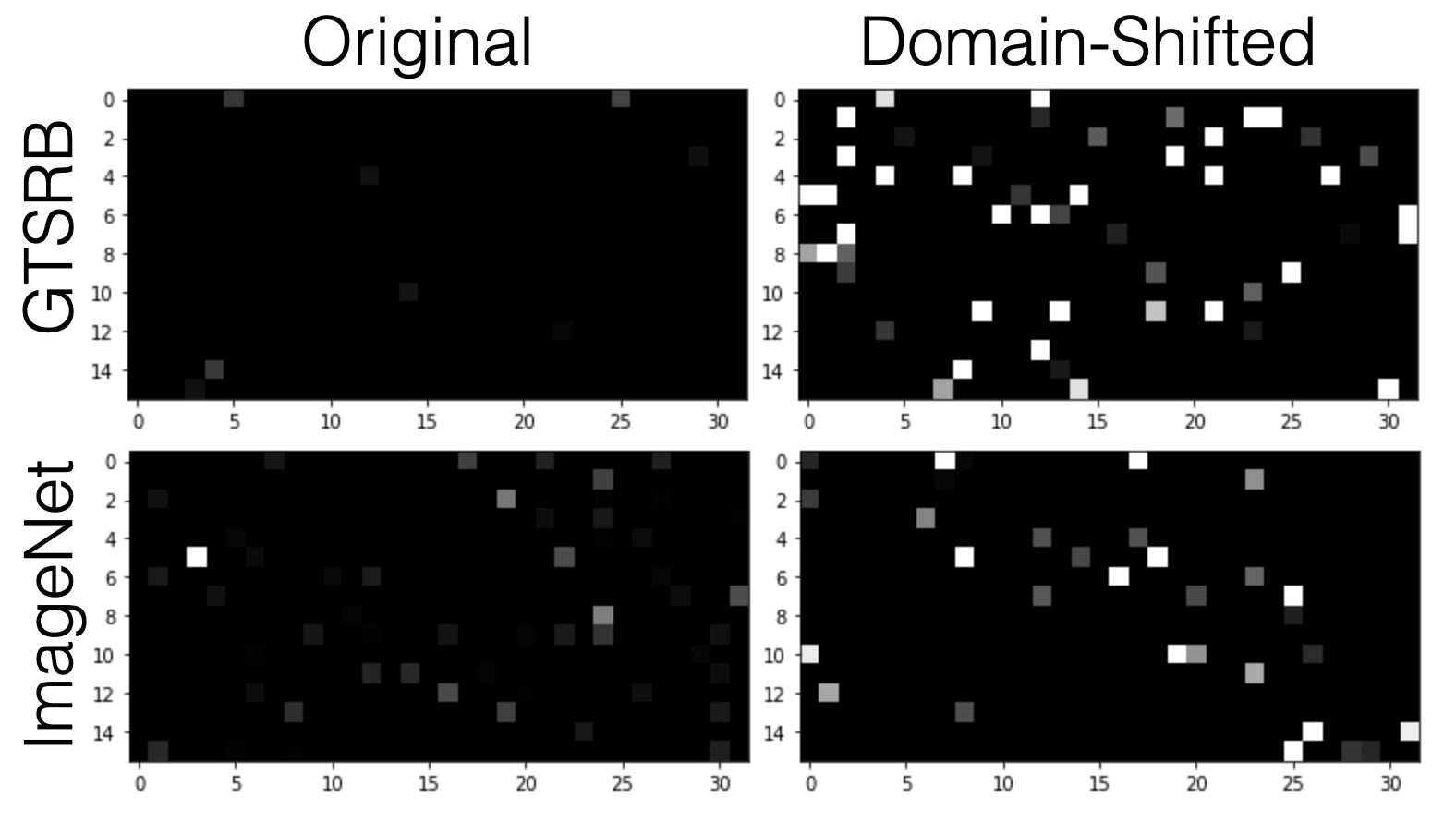}
    \caption{Original and domain-shifted sample in the hidden layer space.}
    \label{fig:feature}
\end{figure}

\section{Conclusion}
\label{sec:Conclusion}
This paper proposes an easy-to-compute  distribution-free upper bound for the distribution overlap index.  The computation of the bound is time-efficient and memory-efficient and only requires finite samples. Two applications of the bound are explored. The first application is for one-class classification. Specifically, this paper introduces a novel distribution-free one-class classifier with the bound being its confidence score function. The classifier is sample-efficient, computation-efficient, and memory-efficient. The proposed classifier is evaluated on novelty detection, out-of-distribution detection,  backdoor detection, and anomaly detection on various datasets and outperforms many state-of-the-art methods. The second application is for domain shift analysis with a proposed theorem to detect the domain shift existence and infer useful data information.  The obtained results show significant promise toward broadening the application of overlap-based metrics.

\bibliography{iclr2023_conference}
\bibliographystyle{IEEEtran}

\newpage
\onecolumn
\appendix
\section{Proof of {\bf Theorem~\ref{thm:bound}}}
\label{app:proof}
\begin{proof}
Let $f_{D^+}$ and $f_{D^-}$ be the probability density functions for $D^+$ and $D^-$. From (\ref{eq:convertible}), we have
     \begin{align}
         \eta(D^+, D^-) &
          = 1 -  \delta_{A}(D^+, D^-) -  \delta_{A^c}(D^+, D^-).
          \label{eq:ineq_eta}
     \end{align}
Using  (\ref{eq:ineq_eta}), triangular inequality, and boundedness, we obtain
     \begin{align}
        ||\mu_{D^+} - \mu_{D^-} || &= ||\int_{B} x \left (f_{D^+}(x) - f_{D^-}(x) \right) dx||
        \le \int_{B} ||x(f_{D^+}(x) - f_{D^-}(x))||dx \\
        &= \int_{A} ||x|| \cdot |f_{D^+}(x) - f_{D^-}(x)|dx
        + \int_{A^c} ||x|| \cdot |f_{D^+}(x) - f_{D^-}(x)|dx \\
        &\le 2r_A \delta_{A} + 2r_{A^c}\delta_{A^c} = 2r_A \delta_{A} + 2r_{A^c}(1-\delta_A - \eta(D^+, D^-))
        \label{eq:ineq_mu}
    \end{align}
  which implies (\ref{eq:theorem}).
    Replacing $r_{A^c}$ with $r_B$ in (\ref{eq:ineq_mu}) implies (\ref{eq:theorem_special}). 
\end{proof}
\section{Proof of {\bf Theorem~\ref{thm:classification_accuracy}}}
\label{app:proof2}
\begin{proof}
    Let $f_D$ and $f_{D^*}$ be their probability density functions,  then
    \begin{align}
        Acc &= \int_{x\sim D^*} \left(p\frac{\min \{f_D(x), f_{D^*}(x) \}}{f_{D^*}(x)}+ q\left(1 - \frac{\min \{f_D(x), f_{D^*}(x)}{f_{D^*}(x)} \right)  \right )f_{D^*}(x) dx \\
        &= p\eta(D, D^*)+q(1-\eta(D, D^*)) = (p-q)\eta(D,D^*)+q \\
        &\le (p-q)(1 - \frac{1}{2r_{B}} ||\mu_{D} - \mu_{D^*}|| - \max_g \frac{r_{B}-r_{A(g)}}{2r_{B}}\left | \mathbb{E}_{D}[g] - \mathbb{E}_{D^*}[g] \right |) +q.
    \end{align}
\end{proof}
\section{Details for Novelty Detection}
\label{app:novelty}
The method in Table~\ref{tab:novelty} is in the same order as shown in Fig.~\ref{fig:errorbar}.
\begin{table}[ht]
    \centering
    \caption{Means and standard deviations of AUROC ($\%$) for different methods on 100 UCI datasets.}
    \begin{tabular}{cccccc}
    \hline
       \multirowcell{1}{Ours} & \multirowcell{1}{L1-Ball} & \multirowcell{1}{K-Center} & \multirowcell{1}{Parzen} &\multirowcell{1}{Gaussian}  & \multirowcell{1}{K-Mean}  \\
        $81.9 \pm 13.8$ & $65.4\pm 25.4$ & $72.2\pm17.6$ & $77.6\pm 18.6$ & $81.1\pm 16.6$ & $76.3 \pm 17.9$ \\
        \hline
      \multirowcell{2}{1-Nearest \\ Neighbor} & \multirowcell{2}{K-Nearest \\ Neighbor}  & \multirowcell{2}{Auto-Encoder \\ Network} & \multirowcell{2}{Linear \\ Programming} & \multirowcell{2}{Principal \\ Component} & \multirowcell{2}{Lof \\ Range} \\
      \\
      $78.8\pm 18.3$ & $79.2 \pm 18.3$ & $73.7\pm 17.9$ & $77.5\pm 18.9$ & $75.1 \pm 18.3$ & $56.5\pm 30.8$ \\
      \hline
      \multirowcell{3}{Nearest \\ Neighbor \\ Distance} & \multirowcell{3}{Minimum \\ Spanning \\ Tree} & \multirowcell{3}{Minimum \\ Covariance \\ Determinant} & \multirowcell{3}{Self \\ Organizing \\ Map} & \multirowcell{3}{Support \\ Vector \\ Machine} & \multirowcell{3}{Minimax \\ Probability \\ Machine} \\
      \\
      \\
      $77.5\pm 16.6$ & $78.8 \pm 17.4$ & $80.4\pm 17.1$ & $77.6\pm 17.4$ & $69.4\pm 24.6$ & $66.4\pm 33.3$ \\
      \hline
      \multirowcell{2}{Mixture \\ Gaussians} & \multirowcell{2}{Local Outlier \\ Factor} & \multirowcell{2}{Naive \\ Parzen} &  
      \multirowcell{2}{Local  Correlation \\ Integral} \\
      \\
      $78.7\pm 18.1$ & $76.1\pm 19.7$ & $78.4\pm 16.2$ & $60.7\pm 23.9$ \\
      \hline
    \end{tabular}
    \label{tab:novelty}
\end{table}
\section{Details for Out-of-Distribution Detection}
\label{app:odd}
Table~\ref{tab:out-of-distribution_CIFAR10} is for CIFAR-10 case and Table~\ref{tab:out-of-distribution_CIFAR100} is for CIFAR-100 case.
\begin{table}[ht]
\caption{Results for CIFAR-10 in-distribution case (higher number implies higher accuracy).  {\bf Boldface} shows the best performing algorithm, whereas \underline{underline} shows the second best algorithm.}
    \centering
    \begin{tabular}{cc|cccc}
    \hline
      Out-of-Distribution Datasets & Method & TPR95 (\%)  & AUROC  (\%) & AUPR (\%)  \\
      \hline
     \multirow{5}{*}{Texture}    & Ours & \underline{64.20} & 92.80 & 92.33 \\
     & MSP & 40.75 & 88.31 & 97.08 \\
     & Mahalanobis & {62.38} & \underline{94.46} & \underline{98.75} \\
     & Energy Score & 47.47 & 85.47 & 95.58 \\
     & GEM & {\bf 72.61} & {\bf 94.59} & {\bf 98.79} \\
     \hline
     \multirow{5}{*}{SVHN}    & Ours & {\bf 94.10} & {\bf 98.56} & {\bf 99.41} \\
     & MSP & 52.41 & 92.11 & 98.32 \\
     & Mahalanobis & \underline{79.34} & \underline{95.72} & \underline{99.04} \\
     & Energy Score & 64.20 & 91.05 & 97.66 \\
     & GEM & 79 & 95.65 & {99.01} \\
     \hline
     \multirow{5}{*}{LSUN-Crop}    & Ours & \underline{83.63} & \underline{96.60} & 96.61 \\
     & MSP & 69.07 & {95.64} & \underline{99.13} \\
     & Mahalanobis & 30.06 & 86.15 & 97.05 \\
     & Energy Score & {\bf 91.89} & {\bf 98.40} & {\bf 99.67} \\
     & GEM & 30.20 & 86.09 & 97.03 \\
     \hline
      \multirow{5}{*}{LSUN-Resize}    & Ours & {\bf 85.41} & {\bf 96.84} & 96.86 \\
     & MSP & 47.45 & 91.30 & \underline{98.11} \\
     & Mahalanobis & 35.64 & 88.12 & 97.45 \\
     & Energy Score & \underline{71.75} & \underline{94.12} & {\bf 98.64} \\
     & GEM & 35.45 & 88.09 & 97.43 \\
     \hline
      \multirow{5}{*}{iSUN}    & Ours & {\bf 79.62} & {\bf 95.64} & 95.68 \\
     & MSP & 43.40 & 89.72 & \underline{97.72} \\
     & Mahalanobis & 26.77 & 87.87 & 97.33 \\
     & Energy Score & \underline{66.27} & \underline{92.56} & {\bf 98.25} \\
     & GEM & 36.80 & 87.85 & 93.33 \\
     \hline
      \multirow{5}{*}{Average Performance}    & Ours & {\bf 81.39} & {\bf 96.08} & 96.17 \\
     & MSP & 50.63 & 91.46 & {\bf 98.07} \\
     & Mahalanobis & 46.83 & 90.46 & 97.92 \\
     & Energy Score & \underline{68.31} & \underline{92.32} & \underline{97.96} \\
     & GEM & 50.81 & 90.45 & 97.91 \\
     \hline
    \end{tabular}
    \label{tab:out-of-distribution_CIFAR10}
\end{table}

\begin{table}[ht]
\caption{Results for CIFAR-100 in-distribution case (higher number implies higher accuracy). {\bf Boldface} shows the best performing algorithm, whereas \underline{underline} shows the second best algorithm.}
    \centering
    \begin{tabular}{cc|cccc}
        \hline
      Out-of-Distribution Datasets & Method & TPR95 (\%)  & AUROC  (\%) & AUPR (\%)  \\
      \hline
     \multirow{5}{*}{Texture}    & Ours & 42.50 & 85.79 & 85.27 \\
     & MSP & 16.71 & 73.58 & 93.02 \\
     & Mahalanobis & {\bf 57.62} & \underline{90.14} & \underline{97.62} \\
     & Energy Score & 20.38 & 76.46 & 93.68 \\
     & GEM & \underline{57.40} & {\bf 90.17} & {\bf 97.63} \\
     \hline
     \multirow{5}{*}{SVHN}    & Ours & {\bf 93.75} & {\bf 98.36} & {\bf 99.36} \\
     & MSP & 15.66 & 71.37 & 92.89 \\
     & Mahalanobis & 51.35 & 89.25 & 97.52 \\
     & Energy Score & 14.59 & 74.10 & 93.65 \\
     & GEM & \underline{51.51} & \underline{89.40} & \underline{97.57} \\
     \hline
     \multirow{5}{*}{LSUN-Crop}    & Ours & \underline{57.76} & \underline{89.95} & 90.03 \\
     & MSP & {33.44} & {83.71} & \underline{96.32} \\
     & Mahalanobis & 1.53 & 58.48 & 89.73 \\
     & Energy Score & {\bf 64.01} & {\bf 93.41} & {\bf 98.59} \\
     & GEM & 1.70 & 58.42 & 89.70 \\
     \hline
      \multirow{5}{*}{LSUN-Resize}    & Ours & {\bf 88.49} & {\bf 97.56} & 97.55 \\
     & MSP & 16.54 & 75.32 & 94.03 \\
     & Mahalanobis & \underline{67.20} & {93.97} & {\bf 98.70} \\
     & Energy Score & 21.38 & 79.29 & 94.97 \\
     & GEM & {67.09} & \underline{94.01} & {\bf 98.70} \\
     \hline
      \multirow{5}{*}{iSUN}    & Ours & {\bf 82.15} & {\bf 96.05} & 96.24 \\
     & MSP & 17.02 & 75.87 & 94.20 \\
     & Mahalanobis & {64.07} & {92.69} & {\bf 98.32} \\
     & Energy Score & 19.20 & 78.98 & 94.90 \\
     & GEM & \underline{64.10} & \underline{92.73} & {\bf 98.32} \\
     \hline
      \multirow{5}{*}{Average Performance}    & Ours & {\bf 72.93} & {\bf 93.53} & 93.69 \\
     & MSP & 19.87 & 75.97 & 94.09 \\
     & Mahalanobis & {48.35} & 84.90 & \underline{96.37} \\
     & Energy Score & 27.91 & 80.44 & 95.15 \\
     & GEM & \underline{48.36} & \underline{84.94} & {\bf 96.38} \\
     \hline
    \end{tabular}
    \label{tab:out-of-distribution_CIFAR100}
\end{table} 

\section{Details for Backdoor Detection}
\label{app:backdoor}
Table~\ref{tab:backdoor} shows the details for backdoor detection.

\begin{table}[ht]
    \centering
    \caption{Comparison results for backdoor detection (higher number implies higher accuracy).}
    \begin{tabular}{ccc|ccccc}
        \hline
      Datasets & Trigger & Metrics (\%) &  Ours & STRIP & Mahalanobis & GEM & MSP  \\
      \hline
     \multirow{3}{*}{MNIST} &   \multirow{3}{*}{All label}  & TPR95  & 83.05 & 2.58  & 50.83 & {\bf 100} & 100 \\
     & & AUROC & {\bf 96.13} & 44.69 & 90.78 & 50.43 & 50 \\
     & & AUPR & {\bf 94.20} & 35.47  & 86.71 & 70.94 & 70.83 \\
     \hline
     \multirow{3}{*}{MNIST} &   \multirow{3}{*}{Naive.1}  & TPR95  & {\bf 100} & 98.85  & 99.86 & 100 & 5.11 \\
     & & AUROC & {\bf 97.50} & 97.32 & 97.49 & 53.95 & 51.64  \\
     & & AUPR & 96.17 & 95.95 & {\bf 96.38} & 74.74 & 50.41 \\
     \hline
     \multirow{3}{*}{MNIST} &   \multirow{3}{*}{Naive.2}  & TPR95  & 96.53 & 67.46 & 35.16 & {\bf 100} & 14.69 \\
     & & AUROC & {\bf 97.28} & 93.67 & 78.63 & 53.51 & 58.14 \\
     & & AUPR & {\bf 95.75} & 89.85 & 78.65 & 74.62 & 64.16 \\
     \hline
     \multirow{3}{*}{CIFAR-10} &   \multirow{3}{*}{TCA.1}  & TPR95  & {\bf 100} & 35.68 & 100 & 100 & 4.38 \\
     & & AUROC & {\bf 97.50} & 83.00 & 97.49 & 50 & 49.23 \\
     & & AUPR & 95.47 & 73.22 & {\bf 97.84} & 76.32 & 52.64 \\
     \hline
     \multirow{3}{*}{CIFAR-10} &   \multirow{3}{*}{TCA.2}  & TPR95  & {\bf 100} & 27.86 & 100 & 100 & 0.02  \\
     & & AUROC & {\bf 97.50} & 68.79 & 97.49 & 50 & 29.90 \\
     & & AUPR & {\bf 97.63} & 72.41 & 95.86 & 67.86 & 18.05 \\
     \hline
     \multirow{3}{*}{CIFAR-10} &   \multirow{3}{*}{Wanet}  & TPR95  &  37.87 & 0.07 & 20.35 & 22.90 & {\bf 100} \\
     & & AUROC & {\bf 92.74} & 34.97 & 50.61 & 57.81 & 50 \\
     & & AUPR & {\bf 89.95} & 37.42 & 57.30 & 68.48 & 74.87 \\
     \hline
     \multirow{3}{*}{GTSRB} &   \multirow{3}{*}{Moving}  & TPR95  & {\bf 99.99} & 54 & \multicolumn{3}{c}{\multirowcell{3}{Fail: dependent data dimensions}} \\
     & & AUROC & {\bf 85.39} & 7.29 &  \\
     & & AUPR & {\bf 96.96} & 89.07 & \\
     \hline
      \multirow{3}{*}{GTSRB} &   \multirow{3}{*}{Filter}  & TPR95  & {\bf 85.39} & 7.29 & \multicolumn{3}{c}{\multirowcell{3}{Fail: dependent data dimensions}} \\
     & & AUROC & {\bf 96.54} & 38.92 &  \\
     & & AUPR & {\bf 95.42} & 38.81 &  \\
     \hline
      \multirow{3}{*}{GTSRB} &   \multirow{3}{*}{Wanet}  & TPR95  &{\bf  100} & 1.24  & 0.51 & 100 & 100 \\
     & & AUROC & {\bf 97.50} & 36.31 & 54.46 & 50 & 50 \\
     & & AUPR & {\bf 97.62} & 39.53 & 48.92 & 75.23 & 75.23 \\
     \hline
      \multirow{3}{*}{YouTube Face} &   \multirow{3}{*}{Sunglasses}  & TPR95  & 73.37 & 83.03 & 71.64 & {\bf 98.58} & 13.06 \\
     & & AUROC &{\bf 95.21}  & 94.80 & 94.38 & 84.29 & 66.55 \\
     & & AUPR & 93.00 & {\bf 95.54} & 94.63 & 88.83 & 53.27 \\
     \hline
      \multirow{3}{*}{YouTube Face} &   \multirow{3}{*}{Lipstick}  & TPR95  & {\bf 96.64} & 90.14 & 90.88 & 94.18 & 3.73 \\
     & & AUROC & {\bf 97.21} & 93.15 & 93.26 & 80.80 & 50.14 \\
     & & AUPR & {\bf 96.30} & 94.98 & 95.09 & 86.53 & 53.27 \\
     \hline
      \multirow{3}{*}{sub-ImageNet} &   \multirow{3}{*}{Invisible}  & TPR95  & {\bf 100} & 7.01  & 0.5 & 100 & 51.40 \\
     & & AUROC & {\bf 97.49} & 66.26 & 4.78 & 50 & 93.61 \\
     & & AUPR & {\bf 96.53}  & 62.83 & 12.27 & 75.26 & 92.46 \\
     \hline
     \multicolumn{2}{c}{\multirow{3}{*}{Average Performance}}    & TPR95  & 89.40 & 39.60 & 56.97 & {\bf 91.57} & 39.24 \\
     & & AUROC & {\bf 96.68} & 70.30 & 75.94 & 58.08 & 54.92 \\
     & & AUPR & {\bf 95.42} & 68.76 & 76.37 & 75.88 & 60.52 \\
     \hline
    \end{tabular}
    \label{tab:backdoor}
\end{table}
\end{document}